\documentclass{article}

\PassOptionsToPackage{numbers}{natbib}

\usepackage[final]{nips_2018}

\usepackage[utf8]{inputenc} 
\usepackage[T1]{fontenc}    
\usepackage{hyperref}       
\usepackage{url}            
\usepackage{booktabs}       
\usepackage{amsfonts}       
\usepackage{nicefrac}       
\usepackage{microtype}      
\usepackage{epsfig}
\usepackage{amsmath}
\usepackage{amsthm}
\usepackage{MnSymbol}
\usepackage{graphicx}
\usepackage{xcolor}
\usepackage{algorithm}
\usepackage{algorithmic}
\usepackage{subcaption}

\title{Local Differential Privacy for Evolving Data}

\author{Matthew Joseph\\
  Computer and Information Science\\
  University of Pennsylvania\\
  \texttt{majos@cis.upenn.edu}
  \And
  Aaron Roth\\
  Computer and Information Science\\
  University of Pennsylvania\\
  \texttt{aaroth@cis.upenn.edu}\\
  \And
  Jonathan Ullman\\
  Computer and Information Sciences\\
  Northeastern University\\
  \texttt{jullman@ccs.neu.edu}\\
  \And
  Bo Waggoner\\
  Computer and Information Science\\
  University of Pennsylvania\\
  \texttt{bowaggoner@gmail.com}\\
}

\newtheorem{theorem}{Theorem}[section]
\newtheorem{lemma}[theorem]{Lemma}
\newtheorem{definition}[theorem]{Definition}

\newtheorem{assm}[theorem]{Assumption}

\newcommand{\E}[1]{\mathbb{E}\left[ #1 \right]}

\renewcommand{\P}[1]{\mathbb{P}\left[ #1 \right]}

\newcommand{\ex}[2]{{\ifx&#1& \mathbb{E} \else \underset{#1}{\mathbb{E}} \fi \left[#2\right]}}
\newcommand{\pr}[2]{{\ifx&#1& \mathbb{P} \else \underset{#1}{\mathbb{P}} \fi \left[#2\right]}}
\newcommand{\var}[2]{{\ifx&#1& \mathrm{Var} \else \underset{#1}{\mathrm{Var}} \fi \left[#2\right]}}

\newcommand{\eps}{\varepsilon}

\newcommand{\trans}{\ensuremath{\mathsf{tr}}}
\newcommand{\algo}{\textsc{Thresh}}
\newcommand{\hhalgo}{\textsc{HeavyThresh}}
\newcommand{\vote}{\textsc{Vote}}
\newcommand{\hhvote}{\textsc{HeavyVote}}
\newcommand{\est}{\textsc{Est}}
\newcommand{\hhest}{\textsc{HeavyEst}}

\newcommand{\cA}{\mathcal{A}}

\newcommand{\cD}{\mathcal{D}}

\newcommand{\cP}{\mathcal{P}}

\newcommand{\cR}{\mathcal{R}}

\newcommand{\thr}[1]{\tau_{#1}}

\newcommand{\last}[1]{f({#1})}
\newcommand{\subpop}[1]{g({#1})}
\newcommand{\change}[1]{\Delta_{#1}}

\newif\ifcomment
\commentfalse

\newif\ifsupp
\supptrue

\ifcomment
\newcommand{\mj}[1]{\textcolor{brown}{[MJ: #1]}}
\newcommand{\ar}[1]{\textcolor{blue}{[AR: #1]}}
\newcommand{\bw}[1]{\textcolor{red}{[BW: #1]}}
\newcommand{\ju}[1]{\textcolor{purple}{[JU: #1]}}
\else
\newcommand{\mj}[1]{}
\newcommand{\ar}[1]{}
\newcommand{\bw}[1]{}
\newcommand{\ju}[1]{}
\fi

\newcommand{\Ber}{\ensuremath{\mathsf{Ber}}}
\newcommand{\Uni}{\ensuremath{\mathsf{Uni}}}

\begin{document}

\maketitle

\begin{abstract}
	There are now several large scale deployments of differential privacy used to
	collect statistical information about users. However, these deployments
	periodically recollect the data and recompute the statistics using algorithms
	designed for a single use. As a result, these systems do not provide
	meaningful privacy guarantees over long time scales. Moreover, existing
	techniques to mitigate this effect do not apply in the ``local model'' of
	differential privacy that these systems use.
	
	In this paper, we introduce a new technique for local differential privacy
	that makes it possible to maintain up-to-date statistics over time, with
	privacy guarantees that degrade only in the number of changes in the
	underlying distribution rather than the number of collection periods.
	We use our technique for tracking a changing statistic in
	the setting where users are partitioned into an unknown collection of groups, and
	at every time period each user draws a single bit from a common (but changing)
	group-specific distribution. We also provide an application to frequency
	and heavy-hitter estimation.
\end{abstract}

\section{Introduction}
After over a decade of research, differential privacy~\citep{DMNS06} is
moving from theory to practice, with notable deployments by
Google~\citep{EPK14, BEMMR+17}, Apple~\citep{A17}, Microsoft~\citep{DKY17}, and
the U.S.~Census Bureau~\citep{A16}. These deployments have revealed gaps
between existing theory and the needs of practitioners. For example, the bulk
of the differential privacy literature has focused on the
\emph{central model}, in which user data is collected by a
trusted aggregator who performs and publishes the results of a differentially
private computation~\citep{DR14}. However, Google, Apple, and Microsoft have
instead chosen to operate in the
\emph{local model}~\citep{EPK14, BEMMR+17, A17, DKY17}, where users
individually randomize their data on their own devices and send it to a
potentially untrusted aggregator for analysis~\citep{KLNRS08}. In addition, the
academic literature has largely focused on algorithms for performing one-time
computations, like estimating many statistical
quantities~\citep{BLR13,RR10,HR10} or training a
classifier~\citep{KLNRS08,CMS11,BST14}. Industrial applications, however have
focused on tracking statistics about a user
population, like the set of most frequently used emojis or words~\citep{A17}.
These statistics evolve over time and so must be re-computed periodically.

Together, the two problems of periodically recomputing a population statistic
and operating in the local model pose a challenge.
Na\"ively repeating a differentially private computation causes
the privacy loss to degrade as the square root of the number of
recomputations, quickly leading to enormous values of $\epsilon$. This na\"ive
strategy is what is used in practice \citep{EPK14, BEMMR+17, A17}. As a
result,~\citet{TKBWW17} discovered that the privacy parameters guaranteed by
Apple's implementation of differentially private data collection can become
unreasonably large even in relatively short time periods.\footnote{Although
the value of $\epsilon$ that Apple guarantees over the
course of say, a week, is not meaningful on its own, Apple does take additional
heuristic steps (as does Google) that make it difficult to combine user data
from multiple data collections~\citep{A17, EPK14, BEMMR+17}. Thus, they may
still provide a strong, if heuristic, privacy guarantee.}
Published research on Google and Microsoft's deployments suggests that they
encounter similar issues~\citep{EPK14, BEMMR+17, DKY17}.

On inspection the na\"ive strategy of regular statistical updates
seems wasteful as aggregate population statistics don't change very
frequently---we expect that the most frequently visited website today will
typically be the same as it was yesterday. However, population statistics
do eventually change, and if we only recompute them infrequently, then we
can be too slow to notice these changes.

The \emph{central model} of differential privacy allows for an elegant solution
to this problem.  For large classes of statistics, we can use the
\emph{sparse vector technique}~\citep{DNRRV09,RR10,HR10,DR14} to repeatedly
perform computations on a dataset such that the error required for a
fixed privacy level grows not with the number of recomputations, but
with the number of times the computation's outcome
\emph{changes significantly}.  For statistics that are relatively stable over
time, this technique dramatically reduces the overall error. Unfortunately, the
sparse vector technique provably has no local analogue~\citep{KLNRS08,U18}.

In this paper we present a technique that makes it possible to repeatedly
recompute a statistic with error that decays with the number of times the
statistic changes significantly, rather than the number of times we recompute
the current value of the statistic, all while satisfying local
differential privacy.  This technique allows for tracking of evolving local
data in a way that makes it possible to quickly detect changes, at modest cost,
so long as those changes are relatively infrequent. Our approach guarantees
privacy under any conditions, and obtains good accuracy by leveraging three
assumptions: (1) each user's data comes from one of $m$ evolving
distributions; (2), these distributions change relatively
infrequently; and (3) users collect a certain amount of data during each
reporting period, which we term an \emph{epoch}.  By varying the lengths of the
epochs (for example, collecting reports hourly, daily, or weekly), we can trade
off more frequent reports versus improved privacy and accuracy.

\subsection{Our Results and Techniques}
Although our techniques are rather general, we first focus our attention on the
problem of privately estimating the average of bits, with one bit held
by each user.  This simple problem is widely applicable because
most algorithms in the local model have the following structure: on each
individual's device, data records are translated into a short bit vector using
sketching or hashing techniques. The bits in this vector are perturbed to
ensure privacy using a technique called \emph{randomized response}, and the
perturbed vector is then sent to a server for analysis.  The server collects
the perturbed vectors, averages them, and produces a data structure encoding
some interesting statistical information about the users as a whole.  Thus many
algorithms (for example, those based on \emph{statistical queries}) can be
implemented using just the simple primitive of estimating the average of bits.

We analyze our algorithm in the following probabilistic model (see
Section~\ref{sec:algo} for a formal description). The population of $n$ users
has an unknown partition into subgroups, each of which has size at least $L$,
time proceeds in rounds, and in each round each user samples a private bit
independently from their subgroup-specific distribution. The private data for
each user consists of the vector of bits sampled across rounds, and our
goal is to track the total population mean over time.  We require that the
estimate be private, and ask for the strong (and widely known) notion of
\emph{local differential privacy}---for every user, no matter how other
users or the server behave, the distribution of the messages sent by that user
should not depend significantly on that user's private data.

To circumvent the limits of local differential privacy, we consider a slightly
relaxed estimation guarantee. Specifically, we batch the rounds into $T$
\emph{epochs}, each consisting of $\ell$ rounds, and aim in each epoch $t$ to
estimate $p^t$, the population-wide mean across the subgroups and rounds of
epoch $t$. Thus, any sufficiently large changes in this mean will be identified
after the current epoch completes, which we think of as introducing a small
``delay".

Our main result is an algorithm that takes data generated according to our
model, guarantees a fixed level of local privacy $\eps$ that grows (up to a
certain point) with the number of distributional changes rather than the number
of epochs, and guarantees that the estimates released at the end of each epoch
are accurate up to error that scales sublinearly in $1/\ell$ and only
polylogarithmically with the total number of epochs $T$.  Our method
improves over the na\"ive solution of simply recomputing the statistic every
epoch -- which would lead to either privacy parameter or error that scales
linearly with the number of epochs---and offers a quantifiable way to reason
about the interaction of collection times, reporting frequency, and accuracy.
We note that one can alternatively phrase our algorithm so as to have a fixed
error guarantee, and a \emph{privacy cost} that scales dynamically with the
number of times the distribution changes\footnote{We can achieve a dynamic,
data-dependent privacy guarantee using the notion of \emph{ex-post}
differential privacy \cite{expost}, for example by using a so-called privacy
odometer \cite{odometer}.}.

\begin{theorem}[Protocol for Bernoulli Means, Informal Version of
	Theorem~\ref{thm:acc}] \label{thm:mainintro}
	In the above model, there is an $\eps$-differentially private local protocol
	that achieves the following guarantee: with probability at least $1-\delta$,
	while the total number of elapsed epochs $t$ where some subgroup distribution
	has changed is fewer than
	$\eps \cdot \min\left(\frac{L}{\sqrt{n\ln(mT/\delta)}},
	\ln(T)\sqrt{\frac{n}{\ell}}\right)$, the protocol outputs estimates
	$\tilde{p}^{t}$ where
	$$
	|\tilde p^t - p^t| = O\left(\ln(T)\sqrt{\frac{\ln(nT/\delta)}{\ell}}\right)
	$$
	where $L$ is the smallest subgroup size, $n$ is the number of users, $\ell$ is
	the chosen epoch length, and $T$ is the resulting number of epochs.
\end{theorem}

To interpret the theorem, consider the setting where there is only one subgroup
and $L = n$. Then to achieve error $\alpha$ we need, ignoring $\log$ factors,
$\ell \geq 1/\alpha^2$ and that fewer than $\eps \alpha\sqrt{n}$ changes have
occured. We emphasize that our algorithm satisfies $\eps$-differential privacy
\emph{for all inputs} without a distributional assumption---only accuracy
relies on distributional assumptions.

Finally, we demonstrate the versatility of our method as a basic building block
in the design of locally differentially private algorithms for evolving data by
applying it to the well-known \emph{heavy hitters} problem. We do so by
implementing a protocol due to~\cite{BS15} on top of our simple primitive.
This adapted protocol enables us to efficiently track the evolution
of histograms rather than single bits. Given a setting in which each user
in each round independently draws an object from a discrete distribution
over a dictionary of $d$ elements, we demonstrate how to maintain
a \emph{frequency oracle} (a computationally efficient representation of a
histogram) for that dictionary with accuracy guarantees that degrade with the
number of times the distribution over the dictionary changes, and only
polylogarithmically with the number of rounds.  We summarize this result below.

\begin{theorem}[Protocol for Heavy-Hitters, Informal Version of
Theorem~\ref{thm:hh_acc}] \label{thm:mainintrohh}
	In the above model, there is an $\eps$-differentially private local protocol
	that achieves the following guarantee: with probability at least $1-\delta$,
	while the total number of elapsed epochs $t$ where some subgroup distribution
	has changed is fewer than $\eps \cdot
	\min\left(\frac{L}{\sqrt{n\ln(mT/\delta)}},
	\ln(T)\sqrt{\frac{n\ln(nT/\delta)}{\ell}}\right)$
	the protocol outputs estimate oracles $\hat f^t$ such that for all $v \in [d]$
	\[
		|\hat f^t(v) - \cP^t(v)| =
		O\left(\ln(T)\sqrt{\frac{\ln(nT/\delta)}{\ell}}
		+ \sqrt{\frac{\ln(2dnT/\delta)}{n}}\right).
	\]
	where $n$ is the number of users, $L$ is the smallest subgroup size,
	$\cP^t$ is the mean distribution over dictionary elements in epoch $t$, $d$ is
	the number of dictionary elements, $\ell$ is the chosen epoch length, and $T$
	is the resulting number of epochs.
\end{theorem}

\subsection{Related Work}
The problem of privacy loss for persistent local statistics has been recognized
since at least the original work of~\citet{EPK14} on RAPPOR (the first
large-scale deployment of differential privacy in the local model).
~\citet{EPK14} offers a heuristic memoization technique that impedes a certain
straightforward attack but does not prevent the differential privacy loss from
accumulating linearly in the number of times the protocol is run.~\citet{DKY17}
give a formal analysis of a similar memoization technique, but the resulting
guarantee is not differential privacy---instead it is a privacy guarantee that
depends on the behavior of other users, and may offer no protection to users
with idiosyncratic device usage. In contrast, we give a worst-case
differential privacy guarantee.

Our goal of maintaining a persistent statistical estimate is similar in spirit
to the model of \emph{privacy under continual observation}~\citet{DNPR10}. The
canonical problem for differential privacy under continual observation is to
maintain a running count of a stream of bits. However, the problem we study is
quite different.  In the continual observation model, new users are arriving,
while existing users' data does not change.  In our model each user
receives new information in each round. (Also, we work in the local
model, which has not been the focus of the work on continual observation.)

The local model was originally introduced by~\citet{KLNRS08}, and the canonical
algorithmic task performed in this model has become frequency estimation (and
heavy hitters estimation). This problem has been studied in a series of
theoretical~\cite{HKR12,BS15,BST17,BNS17,A17} and practical
works~\cite{EPK14,BEMMR+17,A17}. 

\section{Local Differential Privacy}
We require that our algorithms satisfy \emph{local differential privacy}.
Informally, differential privacy is a property of an algorithm $A$, and states
that the distribution of the output of $A$ is insensitive to changes in one
individual user's input. Formally, for every pair of inputs $x,x'$ differing on
at most one user's data, and every set of possible outputs $Z$,
$\P{A(x) \in Z} \leq e^{\eps}\cdot \P{A(x') \in Z}$.  A locally differentially
private algorithm is one in which each user $i$ applies a private algorithm
$A_i$ \emph{only to their data}.

Most local protocols are \emph{non-interactive}: each user $i$ sends a single
message that is independent of all other messages. Non-interactive protocols
can thus be written as
$A(x_1,\dots,x_n) = f(A_1(x_1),\dots,A_n(x_n))$ for some function $f$, where each algorithm $A_i$
satisfies $\eps$-differential privacy. Our model requires an \emph{interactive}
protocol: each user $i$ sends several messages over time, and these may depend on the
messages sent by other users. This necessitates a slightly more complex formalism.

We consider interactive protocols among the $n$ users and an additional center.
Each user runs an algorithm $A_i$ (possibly taking a private input $x_i$) and
the central party runs an algorithm $C$.  We let the
random variable $\trans(A_1,\dots,A_n,C)$ denote the \emph{transcript} containing all the
messages sent by all of the parties.  For a given party $i$ and a set of
algorithms $A_{-i}', C'$, we let $\trans_i(x_i; A_{-i}', C')$ denote the
messages sent by user $i$ in the transcript $\trans(A_{i}(x_i), A_{-i}', C')$.
As a shorthand we will write $\trans_i(x_i)$, since $A_{-i}',C'$ will be clear
from context.  We say that the protocol is locally differentially private if
the function $\trans_i(x_i)$ is differentially private for every user $i$ and
every (possibly malicious) $A_{-i}', C'$.

\begin{definition} \label{def:localdp}
An interactive protocol $(A_1,\dots,A_n,C)$ satisfies
\emph{$\eps$-local differential privacy} if for every user $i$, every pair of
inputs $x_i, x_i'$ for user $i$, and every set of algorithms $A_{-i}',C'$, the
resulting algorithm $\trans_i(x_i) = \trans_i(A_{i}(x_i), A_{-i}', C')$ is
$\eps$-differentially private.  That is, for every set of possible outputs $Z$,
$\P{\trans_i(x_i) \in Z} \leq e^{\eps} \cdot \P{\trans_i(x_i') \in Z}$.
\end{definition}

\section{Overview: The \algo~Algorithm}
\label{sec:algo}
Here we present our main algorithm, \algo.  The algorithmic framework is quite
general, but for this high level overview we focus on the simplest setting
where the data is Bernoulli.  In Section~\ref{sec:bernoulli} we formally
present the algorithm for the Bernoulli case and analyze the algorithm to
prove Theorem~\ref{thm:mainintro}.

To explain the algorithm we first recall the distributional model.
There are $n$ users, each of whom belongs to a subgroup $S_j$ for
some $j \in [m]$; denote user $i$'s subgroup by $\subpop{i}$. There are
$R = T\ell$ rounds divided into $T$ epochs of
length $\ell$, denoted $E^1,\dots,E^T$.  In each round $r$, each user $i$
receives a private bit $x_{i}^{r} \sim \Ber(\mu_{\subpop{i}}^r)$.
We define the population-wide mean by
$\mu^r = \frac{1}{n}(|S_1|\mu_1^r + \ldots + |S_m|\mu_m^r)$.
For each epoch $t$, we use $p^t$ to denote the average of the Bernoulli means
during epoch $t$, $p^t = \frac{1}{\ell} \sum_{r \in E^{t}} \mu^r.$  After every
epoch $t$, our protocol outputs $\tilde{p}^{t}$ such that
$|p^t - \tilde{p}^{t}|$ is small.

The goal of \algo~is to maintain some public \emph{global estimate}
$\tilde{p}^{t}$ of $p^{t}$.  After any epoch $t$, we can \emph{update} this
global estimate $\tilde{p}^{t}$ using \emph{randomized response}: each user
submits some differentially private estimate of the mean of their data, and
the center aggregates these responses to obtain $\tilde{p}^{t}$.  The main idea
of \algo~is therefore to update the global estimate only when it might become
sufficiently inaccurate, and thus take advantage of the possibly small number
of changes in the underlying statistic $p^t$. The challenge is to privately
identify when to update the global estimate.

\medskip
\textbf{The Voting Protocol.}
We identify these ``update needed'' epochs through a \emph{voting protocol}.
Users will examine their data and privately publish a vote for whether they
believe the global estimate needs to be updated. If enough users vote to update
the global estimate, we do so (using randomized response).  The challenge for
the voting protocol is that users must use randomization in their voting
process, to keep their data private, so we can only detect when a large number
of users vote to update.

First, we describe a  na\"ive voting protocol.  In each epoch $t$, each user
$i$ computes a binary \emph{vote} $a_i^t$.  This vote is $1$ if the user
concludes from their own samples that the global estimate $\tilde p^{t-1}$ is
inaccurate, and $0$ otherwise. Each user casts a noisy vote using randomized
response accordingly, and if the sum of the noisy votes is large enough then a
global update occurs.

The problem with this protocol is that small changes in the underlying mean
$p^t$ may cause some users to vote $1$ and others to vote $0$, and this might
continue for an arbitrarily long time without inducing a global update.  As a
result, each voter ``wastes" privacy in every epoch, which is what we wanted to
avoid.  We resolve this issue by having voters also estimate their
\emph{confidence} that a global update needs to occur, and vote
proportionally. As a result, voters who have high confidence
will lose more privacy per epoch (but the need for a global update will be
detected quickly), while voters with low confidence will lose privacy more
slowly (but may end up voting for many rounds).

In more detail, each user $i$ decides their confidence level by comparing
$|\hat{p}^{t} - \hat p_i^{\last{t}}|$---the difference between the local
average of their data in the current epoch and their local average the last
time a global update occurred---to a small set of discrete \emph{thresholds}.
Users with the highest confidence will vote in every epoch, whereas users with
lower confidence will only vote in a small subset of the epochs. We construct
these thresholds and subsets so that in expectation no user votes in more than
a constant number of epochs before a global update occurs, and the amount of
privacy each user loses from voting will not grow with the number of epochs
required before an update occurs. 

\section{\algo: The Bernoulli Case}
\label{sec:bernoulli}

\subsection{The \algo~Algorithm (Bernoulli Case)}
We now present pseudocode for the algorithm \algo, including both the general
framework as well as the specific voting and randomized response procedures. 
We emphasize that the algorithm only touches user data through the subroutines
\vote, and \est, each of which accesses data from a single user in at most two
epochs.  Thus, it is an online local protocol in which user
$i$'s response in epoch $t$ depends only on user $i$'s data from at most two
epochs $t$ and $t'$ (and
the global information that is viewable to all users). \algo~uses carefully
chosen \emph{thresholds} $\thr{b} = 2(b + 1)\sqrt{\ln(12nT/\delta)/2\ell}$ for 
$b = -1, 0, \ldots, \lfloor \log(T) \rfloor$ to discretize the confidence of
each user; see Section~\ref{subsec:accuracy} for details on this choice.

\mj{Note that this assumes we know or have estimated $m$, since $m$ appears
in our setting of $a$ and $b$. Maybe we should point this out and explain why
it's reasonable (seems like the sort of ``market research" that occurs
pre-launch).}

\begin{figure}
\vspace{-40pt}
		\begin{algorithm}[H]
			\caption{Global Algorithm: \algo}
			\begin{algorithmic}[1]
				\REQUIRE number of users $n$, number of epochs $T$, minimum subgroup size
				$L$, number of subgroups $m$, epoch length $\ell$,
				privacy parameter $\eps$, failure parameter $\delta$
				\STATE Initialize global estimate $\tilde p^0 \gets -1$
				\STATE Initialize vote privacy counters $c_1^V, \ldots, c_n^V \gets 0, \ldots, 0$
				\STATE Initialize estimate privacy counters $c_1^E, \ldots, c_n^E \gets 0, \ldots, 0$
				\STATE Initialize vote noise level $a \gets
				\frac{4\sqrt{2n\ln(12mT/\delta)}}
				{L-\frac{3}{\sqrt{2}}\sqrt{n\ln(12mT/\delta)}}$
				\STATE Initialize estimate noise level $b \gets
				\frac{\sqrt{2\ln(12T/\delta)/2n}}
				{\log(T)\sqrt{\ln(12nT/\delta)/2\ell} - \sqrt{\ln(12T/\delta)/2n}}$
				\FOR{each epoch $t \in [T]$}
					\FOR{each user $i \in [n]$}
						\STATE User $i$ publishes $a_i^t \gets \vote(i,t)$
					\ENDFOR
					\STATE GlobalUpdate$^t$ $\gets  \left(\frac{1}{n}\sum_{i=1}^n a_i^t >
					\frac{1}{e^a + 1} + \sqrt{\frac{\ln(10T/\delta)}{2n}}\right)$
					\IF{GlobalUpdate$^t$}
						\STATE $\last{t} \gets t$
						\FOR{each $i \in [n]$}
							\STATE User $i$ publishes $\tilde p_i^t \gets \est(i,t)$
						\ENDFOR
						\STATE Aggregate user estimates into global estimate:
                               $\tilde p^t \gets
                               \frac{1}{n}\sum_{i=1}^n \frac{\tilde p_i^t(e^b + 1) - 1}
                               {e^b - 1}$
					\ELSE
						\STATE $\last{t} \gets \last{t-1}$
						\FOR{each $i \in [n]$}
							\STATE User $i$ publishes $\tilde p_i^t \gets
							\Ber(\frac{1}{e^b + 1})$
						\ENDFOR
						\STATE $\tilde p^t \gets \tilde p^{t-1}$
					\ENDIF
					\STATE Analyst publishes $\tilde p^t$
				\ENDFOR
			\end{algorithmic}
		\end{algorithm}
		
		\vspace{-15pt}
	
		\begin{algorithm}[H]
			\caption{Local Subroutine: \vote}
			\begin{algorithmic}[1]
				\REQUIRE user $i$, epoch $t$
				\STATE Compute local estimate $\hat p_i^t \gets \frac{1}{\ell}\sum_{r \in E^t} x^{r}_i$
				\STATE $b^* \gets$ highest $b$ such that
				$|\hat p_i^t - \hat p_i^{\last{t}}| > \thr{b}$
				\STATE VoteYes$_i^t$ $\gets (c_i^V < \eps/4$ and
				$2^{\lfloor \log T \rfloor - b^*}$ divides $t$)
				\IF{VoteYes$_i^t$}
					\STATE $c_i^V \gets c_i^V + a$
					\STATE $a_i^t \gets \Ber(\frac{e^a}{e^a + 1})$
				\ELSE
					\STATE $a_i^t \gets \Ber(\frac{1}{e^a + 1})$
				\ENDIF
				\STATE Output $a_i^t$
			\end{algorithmic}
		\end{algorithm}
		
		\vspace{-15pt}
	
		\begin{algorithm}[H]
			\caption{Local Subroutine: \est}
			\begin{algorithmic}[1]
				\REQUIRE user $i$, epoch $t$
				\STATE SendEstimate$_i^t$ $\gets \{c_i^E < \eps/4\}$
				\IF{SendEstimate$_i^t$}
					\STATE $c_i^E \gets c_i^E + b$
					\STATE $\tilde p_i^t \gets
						\Ber(\frac{1 + \hat p_i^t(e^b-1)}{e^b + 1})$
				\ELSE
					\STATE $\tilde p_i^t \gets \Ber(\frac{1}{e^b + 1})$
				\ENDIF
				\STATE Output $\tilde p_i^t$
			\end{algorithmic}
		\end{algorithm}
\end{figure}

We begin with a privacy guarantee for \algo. Our proof uses the standard
analysis of the privacy properties of randomized response, combined with the
fact that users have a cap on the number of updates that prevents the privacy
loss from accumulating.  We remark that our privacy proof \emph{does not}
depend on distributional assumptions, which are only used for the proof of
accuracy. We sketch a proof here. A full proof appears in
\ifsupp
Section~\ref{sec:app_b_pri} of the Appendix.
\fi

\begin{theorem}
\label{thm:privacy}
The protocol \algo~satisfies $\eps$-local differential privacy (Definition~\ref{def:localdp})
\end{theorem}

\emph{Proof Sketch}: Na\"ively applying composition would yield a
privacy parameter that scales with $T$.  Instead, we will rely on our defined
privacy ``caps" $c_i^V$ and $c_i^E$ that limit the number of truthful votes and
estimates each user sends. Intuitively, each user sends at most
$O(\tfrac{\eps}{a} + \tfrac{\eps}{b})$ messages that depend on their private
data, and the rest are sampled independently of their private data. Thus, we
need only bound the privacy ``cost" of each of these
$O(\tfrac{\eps}{a} + \tfrac{\eps}{b})$ elements of a user's transcript coming
from a different distribution and bound the sum of the costs by $\eps$. $\hfill \Box$

\subsection{Accuracy Guarantee}
\label{subsec:accuracy}
Our accuracy theorem needs the following assumption on $L$, the size of the
smallest subgroup, to guarantee that a global update occurs whenever any
subgroup has all of its member users vote ``yes".

\begin{assm}
\label{assm:1}
	$L >
	\left(\frac{3}{\sqrt{2}} + \frac{\sqrt{32}}{\eps}\right)
	\sqrt{n\ln(12mT/\delta)}$.
\end{assm}

This brings us to our accuracy theorem, followed by a proof sketch 
(see
\ifsupp
Appendix~\ref{sec:app_b_acc}
\fi
for full details).

\begin{theorem}
\label{thm:acc}
	Given number of users $n$, number of subgroups $m$, smallest subgroup size
	$L$, number of rounds $R$, privacy parameter $\eps$, and chosen epoch length
	$\ell$ and number of epochs $T = R/\ell$, with probability at least
	$1-\delta$, in every epoch $t \in [T]$ such that fewer than
	\[
		\frac{\eps}{4} \cdot
		\min\left(\frac{L}{8\sqrt{2n\ln(12mT/\delta)}}
		- 1,\frac{1}{\sqrt{2}}\left[\log(T)\sqrt{\frac{n}{\ell}} - 1\right]\right)
	\]
	changes have occurred in epochs $1,2, \ldots, t$, \algo~outputs $\tilde p^t$
	such that
	\[
		|\tilde p^t - p^t| \leq
		4(\lfloor \log(T) \rfloor + 2)\sqrt{\frac{\ln(12nT/\delta)}{2\ell}}.
	\]
\end{theorem}
\emph{Proof Sketch}: 
\ifsupp
We begin by proving correctness of the voting process. 
Lemma~\ref{lem:rr_vote_error} guarantees that if
every user decides that their subgroup distribution has not
changed then a global update does not occur, while
Lemma~\ref{lem:vote_yes} guarantees that if every user in some subgroup
decides that a change \emph{has} occurred, then a global update occurs.
By Lemma~\ref{lem:local_error}, for each user
$i$ the individual user estimates driving these voting decisions are
themselves accurate to within
$t_\ell = O(\sqrt{\ln(nT/\delta)/\ell})$ of the true
$\mu_{\subpop{i}}^t$. Finally, by Lemma~\ref{lem:estimate_error} guarantees that if every
user decides that a change has occurred, then a global update occurs that
produces a global estimate $\tilde p^t$ that is within $t_\ell$ of the true
$p^t$.

To reason about how distribution
changes across multiple epochs affect \algo, we use the preceding 
results to show that the number of global updates never exceeds the number of
distribution changes (Lemma~\ref{lem:update}). A more granular guarantee
then bounds the number of changes any user detects---and the number of times
they vote accordingly---as a function of the number of distribution changes
(Lemma~\ref{lem:adj_updates}). These results enable us, in
Lemma~\ref{lem:cap}, to show that each change increases a user's vote privacy
cap $c_i^V$ by at most 2 and estimate privacy cap $c_i^E$ by at most 	1.

Finally, recall that \algo~has each user $i$ compare their current
local estimate $\hat p_i^t$ to their local estimate in the last global update,
$\hat p_i^{\last{t}}$, to decide how to vote, with higher \emph{thresholds}
for 	$|\hat p_i^t - \hat p_i^{\last{t}}|$ increasing the likelihood of a ``yes" vote. This
implies that if every user in some subgroup computes a local estimate
$\hat p_i^t$ such that $|\hat p_i^t - \hat p_i^{\last{t}}|$
exceeds the 	highest 	threshold, then every user sends a ``yes" vote and a
global update occurs, bringing with it the accuracy guarantee of
Lemma~\ref{lem:estimate_error}. In turn, we conclude that $|\tilde p^t - p^t|$
never exceeds the highest threshold, and our accuracy result follows.
\fi $\hfill \Box$

We conclude this section with a few remarks about \algo. First, while the
provided guarantee depends on the number of changes of any size, one can easily
modify \algo~to be robust to changes of size $\leq c$, paying and additive $c$
term in the accuracy. Second, the
accuracy's dependence on $\ell$ offers guidance for its selection: roughly, for
desired accuracy $\alpha$, one should set $\ell = 1/\alpha^2$. Finally, in
practice one may want to periodically assess how many users have exhausted
their privacy budgets, which we can achieve by extending the voting protocol to
estimate the fraction of ``live'' users. We primarily view this as an 
implementation detail outside of the scope of the exact problem we study.

\section{An Application to Heavy Hitters}
\label{sec:hh}
We now use the methods developed above to obtain similar
guarantees for a common problem in local differential privacy known as
\emph{heavy hitters}. In this problem each of $n$ users has their own
dictionary value $v \in \cD$ (e.g. their homepage), and an aggregator wants to
learn the most frequently held dictionary values (e.g. the most common
homepages), known as ``heavy hitters", while satisfying local differential
privacy for each user. The heavy hitters problem has attracted significant
attention~\citep{MS06, HKR12, BST17, BNS17}. Here, we show how our techniques
combine with an approach of~\citet{BS15} to obtain the first guarantees for
heavy hitters on evolving data. We note that our focus on this approach is
primarily for expositional clarity; our techniques should apply just as well to
other variants, which can lead to more efficient algorithms.

\subsection{Setting Overview}
As in the simpler Bernoulli case, we divide time into $\ell \cdot T$ rounds and
$T$ epochs. Here, in each round $r$ each user $i$ draws a sample
$v_i^r$ from a subgroup-specific distribution $\cP_{\subpop{i}}^r$ over the $d$
values in dictionary $\cD$ and track $\cP^1, \ldots, \cP^T$, the weighted
average dictionary distribution in each epoch. We will require the same
Assumption~\ref{assm:1} as in the Bernoulli case, and we also suppose that
$d \gg n$, a common parameter regime for this problem.

In the Bernoulli case users could reason about the evolution of $\mu_j^t$ directly
from their own $\ell$ samples in each epoch. Since it is reasonable to assume
$d \gg \ell$, this is no longer possible in our new setting---$\cP_j^t$ is too
large an object to estimate from $\ell$ samples. However, we can instead
adopt a common approach in heavy hitters estimation and examine a ``smaller"
object using a \emph{hash} on dictionary samples. We will therefore have users
reason about the distribution $p_j^t$ over hashes that $\cP_j^t$ induces, which is
a much smaller joint distribution of $m$ (transformed) Bernoulli distributions.
Our hope is that users can reliably ``detect changes'' by analyzing $p_j^t$,
and the feasibility of this method leans crucially on the properties of the hash
in question.

\subsection{Details and Privacy Guarantee}
\label{sec:hh_privacy}
First we recall the details of the one-shot protocol from~\citet{BS15}. In
their protocol, each user starts with a dictionary value $v \in [d]$ with an
associated basis vector $e_v \in \mathbb{R}^d$. The user hashes this to a
smaller vector $h \in \mathbb{R}^w$ using a (population-wide) $\Phi$, a $w \times d$
Johnson-Lindenstrauss matrix where $w \ll d$. The user then passes this
hash $\hat z_i^t = \Phi e_v$ to their own local randomizer $\cR$, and the
center aggregates these randomized values into a single $\bar z$ which induces
a frequency oracle.

We will modify this to produce a protocol \hhalgo~in the vein of \algo. In each
epoch $t$ each user $i$ computes an estimated histogram $\hat p_i^t$ and then
hashes it into $\Phi \hat p_i^t \in \mathbb{R}^w$, where $w = 20n$ (we assume
the existence of a subroutine \text{GenProj} for generating $\Phi$). Each user
votes on whether or not a  global update has occurred by comparing $\Phi \hat p_i^t$ to their
estimate during the most recent update, $\Phi \hat p_i^{\last{t}}$, in \hhvote. Next,
\hhalgo~aggregates these votes to determine whether or not a global update will
occur. Depending on the result, each user then calls their own estimation
subroutine \hhest~and outputs a randomized response using $\cR$ accordingly.
If a global update occurs, \hhalgo~aggregates these responses into a new
published global hash $\tilde y^t$; if not, \hhalgo~publishes $\tilde y^{t-1}$.
In either case, \hhalgo~publishes $(\Phi, \tilde y^t)$ as
well. This final output is a \emph{frequency oracle}, which for any $v \in [d]$
offers an estimate $\langle \Phi e_v, \tilde y^t\rangle$ of $\cP^t(v)$.

\hhalgo~will use the following thresholds with
$\thr{b} = 2(b + 1)\sqrt{2\ln(16wnT/\delta)/w\ell}$ for
$b = -1, 0, \ldots, \lfloor \log(T) \rfloor$. See
Section~\ref{sec:hh_accuracy} for details on this choice. Fortunately, the bulk
of our analysis uses tools already developed either in
Section~\ref{sec:bernoulli} or ~\citet{BS15}. Our privacy guarantee is almost
immediate: since \hhalgo~shares its voting protocols with \algo, the only
additional analysis needed is for the estimation randomizer $\cR$
\ifsupp
(Lemma~\ref{lem:hh_priv_lemma}).
\fi
Using the privacy of $\cR$, privacy for \hhalgo~follows by the same
proof as for	the Bernoulli case.

\begin{theorem}
\label{thm:hh_priv}
	\hhalgo~is $\epsilon$-local differentially private.
\end{theorem} 

\subsection{Accuracy Guarantee}
\label{sec:hh_accuracy}
As above, an accuracy guarantee for \hhalgo~unfolds along similar lines as 
that for \algo, with additional recourse to results from~\citet{BS15}. We again
require Assumption~\ref{assm:1} and also assume $d = 2^{o(n^2/\ell)}$
(a weak assumption made primarily for neatness in
Theorem~\ref{thm:mainintrohh}).  Our result and its proof sketch follow, with
details and full pseudocode in 
\ifsupp
Appendix Section~\ref{sec:app_hh_acc}.
\fi

\begin{theorem}
\label{thm:hh_acc}
	With probability at least $1-\delta$, in every epoch $t \in [T]$ such that
	fewer than
	\[
		\frac{\eps}{4} \cdot
		\min\left(\frac{L}{8\sqrt{2n\ln(12mT/\delta)}}
		- 1,
		\frac{\log(T)\sqrt{\frac{n\ln(320n^2T/\delta)}{10\ell}}
		- \sqrt{\frac{\ln(16dT/\delta)}{10}}
		- 2\ln(320nT/\delta)\sqrt{\frac{5}{n}}}
		{\sqrt{\ln(320nT/\delta)}
		\left(1 + \frac{20}{\sqrt{n}}\right)}\right)
	\]
	changes have occurred in epochs $1,2, \ldots, t$,
	\[
		|\hat f^t(v) - \cP^t(v)| <
		4(\log(T) + 2)\sqrt{\frac{2\ln(320n^2T/\delta)}{\ell}} 
		+ \sqrt{\frac{\ln(\tfrac{16ndT}{\delta})}{n}}.
	\]
\end{theorem}
\emph{Proof Sketch}: Our proof is similar to that of Theorem~\ref{thm:acc} and proceeds by
proving analogous versions of the same lemmas, with users checking for changes
in the subgroup distribution over observed \emph{hashes} rather than observed
\emph{bits}. This leads to one new wrinkle in our argument: once we show that
the globally estimated hash is close to the true hash, we must translate from
closeness of hashes to closeness of the distributions they induce
\ifsupp
(Lemma~\ref{lem:translation})
\fi
. The rest of the proof, which uses guarantees of
user estimate accuracy to 1. guarantee that sufficiently large changes cause
global updates and 2. each change incurs a bounded privacy loss, largely
follows that of Theorem~\ref{thm:acc}. $\hfill \Box$

\bibliographystyle{plainnat}
\bibliography{persistent_local}

\section{Missing Proofs from Section~\ref{sec:bernoulli}}
\label{sec:app_b_pri}
\begin{theorem}
The protocol \algo~satisfies $\eps$-local differential privacy (Definition~\ref{def:localdp})
\end{theorem}
\begin{proof}
	To begin, we fix an arbitrary private user $i$ and arbitrary algorithms
	$A_{-i}', C'$ for the other users and for the center.  Fix any pair of inputs
	$x_i, x_i'$ for user $i$.  To ease notation, let
	$\trans = \trans_i(A_i(x_i), A_{-i}', C')$ and
	$\trans' = \trans_i(A_i(x_i'), A_{-i}', C')$ be the random variables
	corresponding to the messages sent by user $i$ in the protocol with inputs
	$x_i, x_i'$, respectively. Note that we drop the subscript $i$, since user
	$i$ will be fixed throughout.  To prove the theorem, it suffices to show
	\[
		\frac{\P{\trans = z}}{\P{\trans' = z}} \leq e^{\varepsilon}
	\]
	for every possible set of messages $z$.
	
	The structure of the transcripts is as follows: each epoch $t$ contributes two
	elements, first the vote $a^t$ (the output of $\vote(i,t)$) and the estimate
	$\tilde{p}^{t}$ (the output of $\est(i,t)$).  So we can write
	$z = ((a^1,\tilde{p}^1),\dots,(a^T,\tilde{p}^{T}))$ and
	\begin{align*}
	\label{eq:prod}
		\frac{\P{\trans = z}}{\P{\trans' = z}} 
		={} &\prod_{t=1}^{T}
		\frac{\P{\trans^{t} = (a^t, \tilde{p}^{t}) \mid \trans^{< t} = z^{< t}}}
		{\P{\trans'^{t} = (a^t, \tilde{p}^{t}) \mid \trans'^{< t} = z^{< t}}}.
	\end{align*}

	For any execution of the protocol, we can partition the set of epochs into
	those epochs $S_V \subseteq [T]$ where in at least one of $\trans$ and
	$\trans'$ user $i$ sets VoteYes$_i^t$ to True,
	and those $S_V^c$ where VoteYes$_i^t$ is False in both $\trans$ and $\trans'$;
	similarly, we can partition $[T]$ into those epochs $S_E$ where
	SendEstimate$_i^t$ is True in at least one of $\trans$ and $\trans'$ and
	those $S_E^c$ where SendEstimate$_i^t$ is False in both $\trans$ and
	$\trans'$.
	
	Since every epoch in $S_{V}$ causes the counter $c_i^v$ to increase by $a$,
	$S_{V}$ contains at most $\eps/4a$ epochs from each of $\trans$ and $\trans'$,
	so $|S_V| \leq \eps/2a$.
	
	For any $t \in S^{c}_{V}$, user $i$ will sample $a^{t}$ and $\tilde{p}^{t}$
	from $\Ber(\frac{1}{e^a + 1})$ in both $\trans$ and $\trans'$.  Thus
	$$
		\prod_{t \in S^{c}_{V}} \frac{\P{\trans^{t} = (a^t, \tilde{p}^{t}) \mid \trans^{< t} = z^{< t}}}
		{\P{\trans'^{t} = (a^t, \tilde{p}^{t}) \mid \trans'^{< t} = z^{< t}}} = 1.
	$$

	To complete the proof, we need to bound
	$$
		\prod_{t \in S_{V}} \frac{\P{\trans^{t} = (a^t, \tilde{p}^{t}) \mid \trans^{< t} = z^{< t}}}
		{\P{\trans'^{t} = (a^t, \tilde{p}^{t}) \mid \trans'^{< t} = z^{< t}}} \leq e^{\eps/2},
	$$
	which will hold because every factor in the product is at most $e^a$ and
	$|S_{V}| \leq \eps/2a$. To see why, consider some epoch $t \in S_{V}$.  The
	first component of $\trans^t$ is the vote $a^t \in \{0,1\}$.  The only two
	possibilities for how $a^t$ is chosen are
	$a^t \sim \Ber(\frac{1}{e^a + 1})$ or $a^t \sim \Ber(\frac{e^a}{e^a+1})$. 
	One can easily verify that for any $a^t \in \{0,1\}$,
	$$
		\frac{\P{a^t \mid \trans^{< t} = z^{< t}}}{\P{a^t \mid \trans'^{< t} = z^{< t}}} \leq e^{a}.
	$$

	We now consider the second component of $\trans^t$, which is $\tilde{p}^{t}$.
	As in the $S_V$ case, since every epoch in $S_E$ causes the counter $c_i^E$ to
	increase by $b$, $S_E$ contains at most $\eps/4b$ epochs from each of $\trans$
	and $\trans'$, so $|S_E| \leq \eps/2b$.
	
	When SendEstimate$^t$ is False, then $\tilde{p}^{t}$ is sampled from
	$$
		\Ber\left(\frac{1}{e^b + 1}\right)
	$$
	and when SendEstimate$^t$ is True, then $\tilde{p}^{t}$ is sampled from
	$$
		\Ber\left(\frac{1 + \hat p^t(e^b-1)}{e^b + 1}\right)
	$$
	depending on the value of the private data $\hat{p}^{t}$, which lies in $[0,1]$. 
	Thus, the parameter in the Bernoulli distribution lies in
	$[\frac{1}{e^b+1}, \frac{e^b}{e^b+1}]$.  Again, one can easily verify that for
	any $\tilde{p}^{t} \in \{0,1\}$,
	$$
		\frac{\P{\tilde{p}^{t} \mid \trans^{< t} = z^{< t}, a^{t}}}
		{\P{\tilde{p}^{t} \mid \trans'^{< t} = z^{< t}, a^{t}}} \leq e^{b}.
	$$
	
	Putting it together, we have
	\begin{align*}
		\frac{\P{\trans = z}}{\P{\trans' = z}}
		={} &\prod_{t=1}^{T} \frac{\P{\trans^{t} = (a^t, \tilde{p}^{t}) \mid \trans^{< t} = z^{< t}}}
		{\P{\trans'^{t} = (a^t, \tilde{p}^{t}) \mid \trans'^{< t} = z^{< t}}} \\
		={} &\prod_{t \in S_{V}} \frac{\P{\trans^{t} = a^t \mid \trans^{< t} = z^{< t}}}
		{\P{\trans'^{t} = a^t \mid \trans'^{< t} = z^{< t}}}
		\cdot
		\prod_{t \in S_{E}} \frac{\P{\trans^{t} = \tilde{p}^{t} \mid \trans^{< t} = z^{< t}, a^t}}
		{\P{\trans'^{t} = \tilde{p}^{t}) \mid \trans'^{< t} = z^{< t}, a^t}} \\
		\leq{} &e^{a \cdot |S_V|} \cdot e^{b \cdot |S_E|}
		\leq e^{\eps/2} \cdot e^{\eps/2} \leq  = e^{\eps}.
	\end{align*}
	
	This completes the proof.
\end{proof} 

\section{Missing Proofs From Section~\ref{subsec:accuracy}}
\label{sec:app_b_acc}
We begin the proof of our accuracy guarantee with a series of lemmas. Recalling
that we set
\[
	a = \frac{4\sqrt{2n\ln(12mT/\delta)}}
	{L-\frac{3}{\sqrt{2}}\sqrt{n\ln(12mT/\delta)}}
\]
and
\[
		b = \frac{\sqrt{2\ln(12T/\delta)/2n}}
		{\log(T)\sqrt{\ln(12nT/\delta)/2\ell} - \sqrt{\ln(12T/\delta)/2n}}
\]
we start by showing that if every user votes that a change has not occurred,
then a global update will not occur.

\begin{lemma}
\label{lem:rr_vote_error}
	With probability at least $1-\frac{\delta}{6}$, in every epoch $t \in [T]$,
	if every user $i$ sets VoteYes$_i^t$ $\gets$ False then
	GlobalUpdate$^t \gets$ False.
\end{lemma}
\begin{proof}
	Since every user $i$ sets VoteYes$_i^t \gets$ False, every $a_i^t$ is an iid
	draw from a Bern$\left(\frac{1}{e^a + 1}\right)$ distribution.
	Thus a Chernoff bound says
	\[
		\P{\left|\frac{1}{n} \sum_{i=1}^n a_i^t - \frac{1}{e^a + 1}\right|
		\geq \sqrt{\frac{\ln(12T/\delta)}{2n}}} \leq \frac{\delta}{6T}.
	\]
	Since GlobalUpdate$^t \gets \left(\frac{1}{n}\sum_{i=1}^n a_i^t >
	\frac{1}{e^a + 1} + \sqrt{\frac{\ln(12T/\delta)}{2n}}\right)$,
	GlobalUpdate$^t \gets$ False. Union-bounding across $T$ epochs completes the
	proof.
\end{proof}

Similarly, we also want to ensure that if every user in some subgroup
votes that an update \emph{has} occurred then a global update will indeed
occur.

\begin{lemma}
\label{lem:vote_yes}
	With probability at least $1-\frac{\delta}{3}$, in every epoch $t \in [T]$,
	if there is a subgroup $j$ where every user $i \in S_j$ sets
	VoteYes$_i^t \gets$ True then GlobalUpdate$^t \gets$ True.		
\end{lemma}
\begin{proof}
	Since $|S_j| \geq L$, Chernoff bounds imply that the aggregate vote
	satisfies
	\[
		\frac{1}{n}\sum_{i=1}^n a_i^t > \frac{1}{n}\left[\frac{Le^a}{e^a+1}
		- \sqrt{\frac{L\ln(12mT/\delta)}{2}} + \frac{n-L}{e^a+1}
		- \sqrt{\frac{(n-L)\ln(12mT/\delta)}{2}}\right].
	\]
	Recalling that GlobalUpdate$^t \gets$ True if and only if
	$\frac{1}{n}\sum_{i=1}^n a_i^t
	> \frac{1}{e^a+1} + \sqrt{\frac{\ln(12T/\delta)}{2n}}$, it suffices to
	show that
	\[
		\frac{1}{n}\left[\frac{Le^a}{e^a+1}
		- \sqrt{\frac{L\ln(12mT/\delta)}{2}} + \frac{n-L}{e^a+1}
		- \sqrt{\frac{(n-L)\ln(12mT/\delta)}{2}}\right]
		> \frac{1}{e^a+1} + \sqrt{\frac{\ln(12T/\delta)}{2n}}.		
	\]
	Rearranging, it is enough to show that
	\[
		L > \frac{3}{\sqrt{2}} \cdot
		\frac{e^a + 1}{e^a - 1} \cdot \sqrt{n\ln(12mT/\delta)}
	\]
	and using the fact that $\frac{e^a+1}{e^a-1} < \frac{a+2}{a}$ it is enough
	that
	\[
		a > \frac{3\sqrt{2n\ln(12mT/\delta)}}
		{L-\frac{3}{\sqrt{2}}\sqrt{n\ln(12mT/\delta)}}
	\]
	which follows from our setting of $a$. Union-bounding across $m$
	subgroups and $T$ epochs completes the proof.
\end{proof}

We now show that every user in every epoch obtains an estimate
$\hat p_i^t$ of $\mu_{\subpop{i}}^t$ of bounded inaccuracy. This will enable us
to---among other things---guarantee that users do not send ``false positive" votes.

\begin{lemma}
\label{lem:local_error}
	With probability at least $1-\frac{\delta}{6}$, in each epoch $t \in [T]$ each
	user $i$ has
	\[
		|\hat p_i^t - \mu_{\subpop{i}}^t| < \sqrt{\frac{\ln(12nT/\delta)}{2\ell}}.
	\]
\end{lemma}
\begin{proof}
	$\E{\hat p_i^t} = \mu_{\subpop{i}}^t$, so by an additive Chernoff bound
	\[
		\P{|\hat p_i^t - \mu_{\subpop{i}}^t| \geq \sqrt{\frac{\ln(12nT/\delta)}{2\ell}}} \leq
		2\exp\left(-2\left[\sqrt{\frac{\ln(12nT/\delta)}{2\ell}}\right]^2\ell\right)
		= \delta/6nT.
	\]
	A union bound across $n$ users and $T$ epochs then completes the proof.
\end{proof}

Next, in those epochs in which a global update occurs and no user $i$ has hit
their estimation privacy cap $c_i^E$, in the interest of asymptotic optimality
we want to obtain a similar error for the resulting collated estimate
$\tilde p^t$.

\begin{lemma}
\label{lem:estimate_error}
	With probability at least $1-\frac{\delta}{3}$, in every epoch $t \in [T]$
	where every user $i$ sets SendEstimate$_i^t \gets$ True,
	\[
		\left|p^t - \tilde p^t \right| <
		2(\log(T) + 2)\sqrt{\frac{\ln(12nT/\delta)}{2\ell}}.	
	\]
\end{lemma}
\begin{proof}
	Since every user $i$ sets SendEstimate$_i^t \gets$ True we know that for all $i$
	\[
		\tilde p_i^t \sim
		\Ber\left(\frac{1 + \hat p_i^t(e^b-1)}
		{e^b + 1}\right),
	\]
	so
	\[
		\E{\tilde p^t}
		=
		\E{\frac{1}{n}\sum_{i=1}^n \frac{\tilde p_i^t(e^b + 1) - 1}
		{e^b - 1}}
		=
		\frac{1}{n}\sum_{i=1}^n \frac{\E{\tilde p_i^t}(e^b + 1) - 1}
		{e^b - 1}
		=
		\frac{1}{n} \sum_{i=1}^n \hat p_i^t.
	\]
	Since $\tilde p^t$ is an average of $\{\frac{-1}{e^b-1},
	\frac{e^b}{e^b-1}\}$-valued random variables, we
	transform it into the $\{0,1\}$-valued random variable
	\[
		Y = \left(\tilde p^t + \frac{1}{e^b-1}\right)
		\cdot
		\frac{e^b-1}{e^b+1}.
	\]
	Applying an additive Chernoff bound as above yields
	\[
		\P{\left|Y - \E{Y} \right|
		\geq \sqrt{\frac{\ln(12T/\delta)}{2n}}} \leq \frac{\delta}{6T}
	\]
	which implies that
	\[
		\P{\left|\tilde p^t - \frac{1}{n}\sum_{i=1}^n \hat p_i^t\right|
		\geq \left(\frac{e^b+1}{e^b-1}\right)
		\sqrt{\frac{\ln(12T/\delta)}{2n}}}
		\leq \frac{\delta}{6T}.
	\]
	Similarly, as $\E{\frac{1}{n}\sum_{i=1}^n \hat p_i^t} = p^t$,
	\[
		\P{\left|\frac{1}{n}\sum_{i=1}^n \hat p_i^t - p^t\right|
		\geq \sqrt{\frac{\ln(12T/\delta)}{2n}}}
		\leq \frac{\delta}{6T}.
	\]
	Combining these results in the triangle inequality yields that with
	probability at least $1-\frac{\delta}{6T}$
	\[
		|\tilde p^t - p^t| < 2\left(\frac{e^b+1}{e^b-1}\right)
		\sqrt{\frac{\ln(12T/\delta)}{2n}}.
	\]
	Since $\frac{e^b+1}{e^b-1} < \frac{b+2}{b}$, this implies that
	\[
		|\tilde p^t - p^t| < 2\left(\frac{b+2}{b}\right)
		\sqrt{\frac{\ln(12T/\delta)}{2n}}
	\]
	so to get 
	\[
		|\tilde p^t - p^t| < 
		2(\log(T) + 2)\sqrt{\frac{\ln(12nT/\delta)}{2\ell}}
	\]
	it is enough for
	\[
		b > \frac{\sqrt{2\ln(12T/\delta)/n}}
		{(\log(T) + 2)\sqrt{\ln(12nT/\delta)/2\ell}
		- \sqrt{\ln(12T/\delta)/2n}}.
	\]
	Substituting in our setting of
	\[
		b = \frac{\sqrt{2\ln(12T/\delta)/2n}}
		{\log(T)\sqrt{\ln(12nT/\delta)/2\ell} - \sqrt{\ln(12T/\delta)/2n}}
	\]
	and union-bounding over $T$ epochs completes the proof.
\end{proof}

Finally, we use the above lemmas to reason about how long users' privacy
budgets last. We'll first define a useful term for this claim.
\begin{definition}
	We say a \emph{change} $\change{t}$ occurs in epoch $t \in [T]$ if there exists 
	subgroup $j$ such that $\mu_j^t \neq \mu_j^{t-1}$.
	Given changes $\change{t_1}$ and $\change{t_2}$ where $t_1 < t_2$, we say that $\change{t_1}$
	and $\change{t_2}$ are \emph{adjacent} changes if there does not exist a change
	$\change{t_3}$ such that $t_1 < t_3 < t_2$.
\end{definition}

This lets us prove the following lemma bounding the frequency of global
updates.
\begin{lemma}
\label{lem:update}
	With probability at least $1-\delta$, given adjacent changes $\change{t_1}$ and
	$\change{t_2}$, GlobalUpdate$^t \gets$ True in at most one epoch
	$t \in [t_1,t_2)$.
\end{lemma}
\begin{proof}
	First, with probability at least $1-\delta$ all of the preceding lemma
	in this section apply, and we condition on them for the remainder
	of this proof.		
	
	Assume instead that GlobalUpdate$^t \gets$ True and
	GlobalUpdate$^{t'} \gets$ True as well for $t_1 \leq t < t' \leq t_2-1$, and
	that GlobalUpdate$^{t_3} \gets$ False for all $t_3 \in (t,t')$. Recall that by
	Lemma~\ref{lem:rr_vote_error}, if in epoch $t'$ every user $i$ sets
	VoteYes$_i^{t'} \gets$ False then
	\[
		\frac{1}{n}\sum_{i=1}^n a_i^{t'} \leq
		\frac{1}{e^{a} + 1} + \sqrt{\frac{\ln(12T/\delta)}{2n}}
	\]
	which means GlobalUpdate$^{t'} \gets$ False. Therefore since we know
	GlobalUpdate$^{t'} \gets$ True, it follows that at least one user $i$ sets
	VoteYes$_i^{t'} \gets$ True. By the thresholding structure of \algo, this
	implies that
	\[
		|\hat p_i^{t'} - \hat p_i^t| > 2\sqrt{\frac{\ln(12nT/\delta)}{2\ell}}.
	\]
	Since Lemma~\ref{lem:local_error} guarantees that both $\hat p_i^{t'}$ and
	$\hat p_i^t$ are within $\sqrt{\frac{\ln(12nT/\delta)}{2\ell}}$ of
	$\mu_{\subpop{i}}^{t'}$ and $\mu_{\subpop{i}}^t$ respectively, it follows that
	$\mu_{\subpop{i}}^{t'} \neq \mu_{\subpop{i}}^t$. This contradicts
	the fact that $\change{t_1}$ and $\change{t_2}$ were adjacent changes.
\end{proof}

We similarly bound the frequency with which users vote that a change has
occurred.

\begin{lemma}
\label{lem:adj_updates}
	With probability at least $1-\delta$, given adjacent changes $\change{t_1}$ and
	$\change{t_2}$ such that a global update occurs in $t_3 \in [t_1, t_2)$, for each 
	user $i$ there is at most one epoch $t \in (t_3, t_2)$ where
	VoteYes$_i^t \gets$ True.
\end{lemma}
\begin{proof}
	First, with probability at least $1-\delta$ all of the preceding lemmas in
	this section apply, and we condition on them for the remainder
	of this proof. In particular, Lemma~\ref{lem:update} implies that $t_3$ is the
	only epoch in $[t_1, t_2)$ in which a global update occurs.
	
	For contradiction, let $t_3 < t_4 < t_5 < t_2$ and assume that user $i$ sets
	VoteYes$_i^{t_4} \gets$ True and VoteYes$_i^{t_5} \gets$ True.
	Since there is only one $t \in [T-1]$ such that
	$2^{\lfloor \log(T) \rfloor - 1}$ divides $t$, and the construction of the
	Vote subroutine requires this for a user $m$ to set VoteYes$_m^t \gets$ True,
	without 	loss of generality we may suppose that
	$|\hat p_i^{t_4} - \hat p_i^{\last{t_4}}| > T_b$ and
	$|\hat p_i^{t_5} - \hat p_i^{\last{t_5}}| > T_{b'}$ where $b > b' \geq 1$.
	However, Lemma~\ref{lem:local_error} then implies that \emph{every} user $m$
	in $S_{\subpop{i}}$ has
	$|\hat p_m^{t_5} - \hat p_i^{\last{t_5}}| > T_b'$, so by
	Lemma~\ref{lem:update} GlobalUpdate$^{t_5} \gets$ True, a contradiction of
	$t_3$ being the only epoch in $[t_1, t_2)$ in which a global update occurs.
\end{proof}

Our last lemma before our main theorem ties the above results together to
relate changes to increases in users' privacy caps $c_i$. This will eventually
let us lower bound the time for which \algo~outputs accurate results.

\begin{lemma}
\label{lem:cap}
	Denote by $c_i^t$ the value of $c_i$ in epoch $t$. Then with probability at
	least $1-\delta$, across all epochs, given any two adjacent changes $\change{t_1}$
	and $\change{t_2}$, for every user $i$
	\[
		c_{i,E}^{t_2-1} \leq c_{i,E}^{t_1-1} + 1
	\]
	and
	\[
		c_{i,V}^{t_2-1} \leq c_{i,V}^{t_1-1} + 2.
	\]
\end{lemma}
\begin{proof}
	First, with probability at least $1-\delta$ all of the preceding lemmas in
	this section apply, and we condition on them for the remainder of this proof.	
	
	Fix a user $i$. First, $c_{i,E}$ increases in any epoch $t$ where
	SendEstimate$_i^t \gets$ True. This only happens in epoch where
	GlobalUpdate$^t \gets$ True, and by Lemma~\ref{lem:update}, at most one global
	update occurs in epochs in $[t_1, t_2)$, so
	$c_{i,E}^{t_2-1} \leq c_{i,E}^{t_1-1} + 1$. We analyze $c_{i,V}$ in two cases.
	
	\underline{Case 1}: For all $t \in [t_1, t_2)$, GlobalUpdate$^t \gets$
	False. Here, since no global update occurs, if $c_{i,V}^{t_2-1} > c_i^{t_1} + 2$
	then there must exist 3 epochs $t \in [t_1,t_2)$ where user $i$ sets
	VoteYes$_i^t \gets$ True, a contradiction of Lemma~\ref{lem:adj_updates}.
	
	\underline{Case 2}: For some epoch $t^* \in [t_1, t_2)$,
	GlobalUpdate$^{t^*} \gets$ True. It then suffices to show that user $i$ sets
	VoteYes$_i^t \gets$ True in at most two epochs $t \in [t_1,t_2-1]$ (possibly
	including $t^*$).
	
	Assume instead that VoteYes$_i^{t_3}$, VoteYes$_i^{t_4}$, and
	VoteYes$_i^{t_5} \gets$ True for distinct $t_3, t_4, t_5 \in [t_1,t_2)$. By
	Lemmas~\ref{lem:local_error} and~\ref{lem:estimate_error}, VoteYes$_i^t \gets$
	False in any epoch $t \in [t^*+1,t_2)$. Therefore
	$t_3, t_4, t_5 \in [t_1, t^*]$, and at least two are in $[t_1, t^*-1]$ and do
	not trigger a global update. This again contradicts
	Lemma~\ref{lem:adj_updates}.
\end{proof}

Taken together, these preliminary results let us prove our main accuracy
theorem.

\begin{theorem}
	With probability at least $1-\delta$, in every epoch $t \in [T]$ such that
	fewer than
	\[
		\frac{\eps}{4} \cdot
		\min\left(\frac{L}{8\sqrt{2n\ln(12mT/\delta)}}
		- 1,\frac{1}{\sqrt{2}}\left[\log(T)\sqrt{\frac{n}{\ell}} - 1\right]\right)
	\]
	changes have occurred in epochs $1,2, \ldots, t$,
	\[
		|\tilde p^t - p^t| \leq
		4(\lfloor \log(T) \rfloor + 2)\sqrt{\frac{\ln(12nT/\delta)}{2\ell}}.
	\]
\end{theorem}
\begin{proof}
	First, with probability at least $1-\delta$ all of the preceding lemmas and 
	corollaries in this section apply, and we condition on them for the remainder
	of this proof. In particular, since
	\begin{align*}
		\min\left(\frac{\eps}{8a}, \frac{\eps}{4b}\right) =&\;
		\frac{\eps}{4} \cdot \min\left(\frac{1}{2a}, \frac{1}{b}\right) \\
		=&\; \frac{\eps}{4} \cdot
		\min\left(\frac{L-\frac{3}{\sqrt{2}}\sqrt{n\ln(12mT/\delta)}}
		{8\sqrt{2n\ln(12mT/\delta)}},
		\frac{\log(T)\sqrt{\ln(12nT/\delta)/2\ell} - \sqrt{\ln(12T/\delta)/2n}}
		{\sqrt{2\ln(12T/\delta)/2n}}\right) \\
		>&\; \frac{\eps}{4} \cdot
		\min\left(\frac{L}{8\sqrt{2n\ln(12mT/\delta)}}
		- 1,
		\frac{1}{\sqrt{2}}\left[\log(T)\sqrt{\frac{n}{\ell}} - 1\right]\right)
	\end{align*}
	we know that the number of changes thus far is less than
	$\min\left(\frac{\eps}{8a}, \frac{\eps}{4b}\right)$, and by
	Lemma~\ref{lem:cap}	for every user $i$, $c_i^V < \eps/4$ and $c_i^E < \eps/4$,
	i.e. no user has exceeded their voting or estimation privacy budget.
	
	Now suppose for contradiction that
	\[
		|\tilde p^t - p^t| >
		4(\lfloor \log(T) \rfloor + 2)\sqrt{\frac{\ln(12nT/\delta)}{2\ell}}.
	\]
	By Lemma~\ref{lem:estimate_error} this means GlobalUpdate$^t \gets$ False,
	so by Lemma~\ref{lem:vote_yes} for every subgroup $j \in [m]$ there
	exists user $i \in S_j$ such that
	\[
		|\hat p_i^t - \hat p_i^{\last{t}}|
		\leq 2(\lfloor \log(T) \rfloor + 1)\sqrt{\frac{\ln(12nT/\delta)}{2\ell}}.
	\]
	Lemma~\ref{lem:local_error} then implies that
	\[
		|\mu_j^t - \mu_j^{\last{t}}|
		\leq 2(\lfloor \log(T) \rfloor + 2)\sqrt{\frac{\ln(12nT/\delta)}{2\ell}}.
	\]
	Since this holds for every subgroup $j$, we get that
	\[
		|p^t - p^{\last{t}}|
		\leq 2(\lfloor \log(T) \rfloor + 2)\sqrt{\frac{\ln(12nT/\delta)}{2\ell}},
	\]
	and since GlobalUpdate$^t \gets$ False, by Lemma~\ref{lem:estimate_error}
	this means that $\tilde p^t = \tilde p^{\last{t}}$ and
	\[
		|\tilde p^t - p^t|
		\leq 4(\lfloor \log(T) \rfloor + 2)\sqrt{\frac{\ln(12nT/\delta)}{2\ell}}.
	\]
	a contradiction.
\end{proof}

\section{Missing Proofs from Section~\ref{sec:hh_privacy}}
\label{sec:app_hh_pri}
We start with full pseudocode for \hhalgo.

\begin{figure}
\vspace{-40pt}
	\begin{algorithm}[H]
		\caption{Global Algorithm: \hhalgo}
		\begin{algorithmic}[1]
			\REQUIRE number of users $n$, number of epochs $T$, minimum subgroup size
			$L$, number of subgroups $m$, epoch length $\ell$,
			privacy parameter $\eps$, failure parameter $\delta$
			\STATE Initialize global estimate $\tilde y^0 \gets -1$
			\STATE Initialize update counters $c_1, \ldots, c_n \gets 0, 0, \ldots, 0$
			\STATE Initialize vote noise level $a \gets
			\frac{4\sqrt{2n\ln(12mT/\delta)}}
			{L-\frac{3}{\sqrt{2}}\sqrt{n\ln(12mT/\delta)}}$
			\STATE Initialize estimate noise level $b \gets
			\frac{2\left(\sqrt{\frac{\ln(16wT/\delta)}{nw}}
			+ \frac{\ln(16wT/\delta)\sqrt{w}}{n^2}\right)}
			{2(\log(T)+2)\sqrt{\frac{2\ln(16wnT/\delta)}{w\ell}}
			- 2\sqrt{\frac{\ln(16dT/\delta)}{2wn}}
			- \frac{\ln(16wT/\delta)\sqrt{w}}{n^2}}$
	        \STATE $w \gets 20n$
			\STATE Initialize JL matrix $\Phi \gets \text{GenProj}(w, d)$
			\FOR{each epoch $t \in [T]$}
				\FOR{each user $i \in [n]$}
					\STATE User $i$ publishes $a_i^t \gets \hhvote(i,t)$
				\ENDFOR
				\STATE GlobalUpdate$^t$ $\gets \left(\frac{1}{n}\sum_{i=1}^n a_i^t >
				\frac{1}{e^a + 1} + \sqrt{\frac{\ln(16T/\delta)}{2n}}\right)$
				\IF{GlobalUpdate$^t$}
					\STATE $\last{t} \gets t$
					\FOR{each $i \in [n]$}
						\STATE User $i$ publishes $\tilde z_i^t \gets \hhest(i,t)$
					\ENDFOR
					\STATE Aggregate user estimates into global estimate: \\
                               $\tilde y^t \gets \frac{1}{n}\sum_{i=1}^n \tilde z_i^t$
				\ELSE
					\STATE $\last{t} \gets \last{t-1}$
					\FOR{each $i \in [n]$}
						\STATE User $i$ publishes $\tilde z_i^t \gets \cR(0, b)$
					\ENDFOR
					\STATE $\tilde y^t \gets \tilde y^{t-1}$
				\ENDIF
				\STATE Analyst publishes $\tilde y^t$
				\STATE Analyst publishes FO$^t \gets (\Phi, \tilde y^t)$
			\ENDFOR
		\end{algorithmic}
	\end{algorithm}
	
	\vspace{-15pt}

	\begin{algorithm}[H]
		\caption{Local Subroutine: \hhvote}
		\begin{algorithmic}[1]
			\REQUIRE user $i$, epoch $t$
			\STATE Compute local estimate
			$\hat p_i^t \gets \frac{1}{\ell} \sum_{r=(t-1)\ell + 1}^{t\ell} v_i^r$				 
			\STATE Compute local hash $\hat y_i^t \gets \Phi \hat p_i^t$
			\STATE $b^* \gets$ highest $b$ such that $||\hat y_i^t - \hat y_i^{\last{t}}||_\infty > \thr{b}$
			\STATE VoteYes$_i^t$ $\gets (c_i^V < \eps/4a$ and
			$2^{\lfloor \log T \rfloor - b^*}$ divides $t$)
			\IF{VoteYes$_i^t$}
				\STATE $c_i^V \gets c_i^V + a$
				\STATE $a_i^t \gets \Ber(\frac{e^a}{e^a + 1})$
			\ELSE
				\STATE $a_i^t \gets \Ber(\frac{1}{e^a + 1})$
			\ENDIF
			\STATE Output $a_i^t$
		\end{algorithmic}
	\end{algorithm}
	
	\vspace{-15pt}
	
	\begin{algorithm}[H]
		\caption{Frequency Oracle: $\cA_{FO}$}
		\begin{algorithmic}[1]
			\REQUIRE Frequency oracle $(\Phi, \frac{1}{n}\sum_{i=1}^n z_i)$, 
			dictionary value to be estimated $v \in [d]$
	        \STATE Output $\hat f(v) = \langle \Phi e_v, \bar z \rangle$
		\end{algorithmic}
	\end{algorithm}
\end{figure}

\begin{figure}[h]
\vspace{-40pt}
	\begin{algorithm}[H]
		\caption{Local Subroutine: \hhest}
		\begin{algorithmic}[1]
			\REQUIRE user $i$, epoch $t$
			\STATE SendEstimate$_i^t$ $\gets \{c_i^E < \eps/4b\}$
			\IF{SendEstimate$_i^t$}
				\STATE $c_i \gets c_i^E + b$
				\STATE $\tilde z_i^t \gets \cR(\hat y_i^t, b)$
			\ELSE
				\STATE $\tilde z_i^t \gets \cR(0, b)$
			\ENDIF
			\STATE Output $\tilde z_i^t$
		\end{algorithmic}
	\end{algorithm}
	
	\vspace{-15pt}
	
	\begin{algorithm}[H]
	\caption{Client Randomizer: $\cR$}
		\begin{algorithmic}[1]
			\REQUIRE Hashed histogram $h = \Phi \hat p_i^t$, privacy
			parameter $\eps$
			\STATE Sample $j \in [w]$ uniformly at random
			\STATE $c_\varepsilon \gets
			\frac{e^b+1}{e^b-1}$
			\STATE $z \gets 0 \in \mathbb{R}^w$
	        \IF{$h \neq 0$}
	        		\STATE Draw $x \sim \Uni(0,1)$
	        		\IF{$x < \frac{1}{2} + \frac{h_j\sqrt{w}}{2c_\varepsilon}$}
	        			\STATE $z_j \gets c_\varepsilon \sqrt{w}$
	        		\ELSE
	        			\STATE $z_j \gets -c_\varepsilon \sqrt{w}$
	        		\ENDIF
	        	\ELSE
	        		\STATE $z_j \gets c_\varepsilon \sqrt{w} \text{ or } 
	        		-c_\varepsilon \sqrt{w}$ u.a.r
	        	\ENDIF
	        	\STATE Output $z$
		\end{algorithmic}
	\end{algorithm}
	
	\vspace{-20pt}
\end{figure}

Next, we prove a lemma guaranteeing the privacy of the $\cR$ subroutine.

\begin{lemma}
\label{lem:hh_priv_lemma}
		$\cR$ is $\varepsilon$-locally DP.
\end{lemma}
\begin{proof}
		Choose a possible output $z$ of $\cR$. Let $h_1$ and $h_2$ be two
		arbitrary input hashes. It suffices to show
		\[
			\frac{\P{\cR(h_1) = z}}{\P{\cR(h_2) = z}} \leq e^{\varepsilon}.
		\]
		\underline{Case 1}: $h_1$ and $h_2$ are zero vectors. Then $\cR(h_1)$ and
		$\cR(h_2)$ 	have identical output distributions and the result is immediate.
		
		\underline{Case 2:} Exactly one	(WLOG $h_1$) hash is a nonzero vector. Then
		\[
			\P{\cR(h_2) = z} = \frac{1}{2w}.
		\]
		Similarly,
		\[
			\P{\cR(h_1) = z} \leq \frac{1}{w} \cdot \left(\frac{1}{2}
			+ \frac{1}{2c_\varepsilon}\right).
		\]
		Therefore
		\[
			\frac{\P{\cR(h_1) = z}}{\P{\cR(h_2) = z}} \leq 1 + \frac{1}{c_\varepsilon}
			= 1 + \frac{e^\varepsilon - 1}{e^\varepsilon + 1} \leq e^\varepsilon.
		\]
		
		\underline{Case 3:} Neither $h_1$ nor $h_2$ is a zero vector. Then by the
		logic above
		\[
			\frac{\P{\cR(h_1) = z}}{\P{\cR(h_2) = z}} \leq
			\frac{1 + \frac{e^\varepsilon - 1}{e^\varepsilon + 1}}
			{1 - \frac{e^\varepsilon - 1}{e^\varepsilon + 1}}
			= \frac{e^\varepsilon + 1 + e^\varepsilon - 1}
			{e^\varepsilon + 1 - e^\varepsilon + 1} = e^\varepsilon.
		\]		
\end{proof}

\section{Missing Proofs From Section~\ref{sec:hh_accuracy}}
\label{sec:app_hh_acc}
First, recall that we set
\[
	a = \frac{4\sqrt{2n\ln(12mT/\delta)}}
	{L-\frac{3}{\sqrt{2}}\sqrt{n\ln(12mT/\delta)}}
\]
and
\[
	b = \frac{2\left(\sqrt{\frac{\ln(16wT/\delta)}{nw}}
	+ \frac{\ln(16wT/\delta)\sqrt{w}}{n^2}\right)}
	{2(\log(T)+2)\sqrt{\frac{2\ln(16wnT/\delta)}{w\ell}}
	- 2\sqrt{\frac{\ln(16dT/\delta)}{2wn}}
	- \frac{\ln(16wT/\delta)\sqrt{w}}{n^2}}
\]

We start with a result about $\cR$.

\begin{lemma}
\label{lem:unbiased}
	For all $ \epsilon > 0$ and
	$x \in [-\frac{1}{\sqrt{w}}, \frac{1}{\sqrt{w}}]^w$,
	$\E{\cR(x, \varepsilon)} = x$.
\end{lemma}
\begin{proof}
	In the case where $x=0$, we get
	\begin{align*}
			\E{\cR(x)}_j =&\;
			\frac{1}{w} \cdot \left(-\frac{c_\varepsilon \sqrt{w}}{2}
			+ \frac{c_\varepsilon \sqrt{w}}{2}\right) = 0,
	\end{align*}
	and for $x \neq 0$
	\begin{align*}
			\E{\cR(x)}_j =&\;
			\frac{1}{w}\left[
			\left(\frac{1}{2} + \frac{x_j\sqrt{w}}{2c_\varepsilon} \right)
			c_\varepsilon \sqrt{w}
			+ \left(\frac{1}{2} - \frac{x_j\sqrt{w}}{2c_\varepsilon}\right)
			(-c_\varepsilon \sqrt{w})\right] = x_j.
	\end{align*}
\end{proof}

Lemmas~\ref{lem:rr_vote_error} and~\ref{lem:vote_yes}, since they cover
portions of the voting process shared between \vote~and \hhvote, apply here
with only their failure probabilities changed to $\delta/8$ and $\delta/4$. We
start with an analogue of Lemma~\ref{lem:local_error}.

\begin{lemma}
\label{lem:hh_local_error}
	With probability at least $1-\frac{\delta}{8}$, for every epoch $t$ and user
	$i$, defining by $p_{\subpop{i}}^t$ the $d$-dimensional vector where
	$p_{\subpop{i}}^t(q) = \cP_{\subpop{i}}^t(q)$,
	\[
		\left\Vert
		\Phi\hat p_i^t - \Phi p_{\subpop{i}}^t
		\right\Vert_{\infty}
		< \sqrt{\frac{2\ln(16wnT/\delta)}{w\ell}}.
	\]	
\end{lemma}
\begin{proof}
	$\Phi \hat p_i^t$ is a vector with entries in $\pm \tfrac{1}{\sqrt{w}}$, so
	setting
	$X = \frac{\sqrt{w}\left(\Phi \hat p_i^t + \frac{1}{\sqrt{w}}\right)}{2}$ we
	get 	$X \in [0,1]^m$ where each index $X_j$ has
	$\E{X_j} = \frac{\sqrt{w}\left((\Phi p_{\subpop{i}}^t)_j + \frac{1}{\sqrt{w}}\right)}{2}$. A
	Chernoff bound then says that, with probability at least $1-\frac{\delta}{8}$,
	for every user $i$ and every epoch $t$
	\[
		\left\Vert X - \E{X_j}\right\Vert_{\infty}
		< \sqrt{\frac{\ln(16wnT/\delta)}{2\ell}}.
	\]
	Scaling this result by $\frac{2}{\sqrt{w}}$ and transforming $X$ back into
	$\Phi \hat p_i^t$ yields the claim.
\end{proof}

This brings us to an analogue of Lemma~\ref{lem:estimate_error}.

\begin{lemma}
\label{lem:hh_estimate_error}
	With probability at least $1-\frac{3\delta}{8}$, for every epoch $t$ where
	every user $i$ sets SendEstimate$_i^t \gets$ True,
	\[
		\left\Vert
		\tilde y^{t} - \Phi p^t
		\right\Vert_\infty <
		2(\log(T)+2)\sqrt{\frac{2\ln(16wnT/\delta)}{w\ell}}.
	\]
\end{lemma}
\begin{proof}
		By Lemma~\ref{lem:unbiased},
		$\E{\tilde y^t} = \frac{1}{n} \sum_{i=1}^n \Phi \hat p_i^t$, and we want to 
		begin by bounding
		$\left\Vert \tilde y^{t}
		- \frac{1}{n} \sum_{i=1}^n \Phi \hat p_i^{t}\right\Vert$.
		First, since each of the $n$ random variables $\hat z_i^t$ that make up
		$\tilde y^t = \frac{1}{n} \sum_{i=1}^n \hat z_i^t$ is a zero vector except
		for an independently randomly chosen index $s \in [w]$, for each $s \in [w]$
		we can bound the number $N_j^t$ of vectors $\hat z_i^t$ that are nonzero in
		index $s$ by an additive Chernoff bound:
		\[
			\P{N_s^t > \frac{n}{w} + \sqrt{\frac{n\ln(8wT/\delta)}{2}}}
			\leq \frac{\delta}{8wT}.
		\]
		Union bounding over $w$ indices, since $\tilde y^t$ is normalized by $1/n$,
		we get that
		\[
			\tilde y^t \in 
			\left[-c_\varepsilon\sqrt{w}\left(\frac{1}{w}
		+ \sqrt{\frac{\ln(8wT/\delta)}{2n}}\right), c_\varepsilon\sqrt{w}\left(\frac{1}{w}
		+ \sqrt{\frac{\ln(8wT/\delta)}{2n}}\right)\right]^w.
		\]
		Thus scaling, applying a Chernoff bound to each index, then re-scaling and
		union bounding over all $w$ indices and $T$ epochs gives us that with 
		probability at least $1-\frac{\delta}{8}$ in every epoch $t$ where
		GlobalUpdate$^t \gets$ True
		\begin{align*}
			\left\Vert
			\tilde y^{t} - \frac{1}{n} \sum_{i=1}^n \Phi \hat p_i^{t}
			\right\Vert_\infty
			<&\; 2c_\varepsilon\left(\frac{1}{w}
			+ \sqrt{\frac{\ln(8wT/\delta)}{2n}}\right)
			\sqrt{\frac{w\ln(16wT/\delta)}{2n}} \\
			<&\; c_\varepsilon\left(
			\sqrt{\frac{2\ln(16wT/\delta)}{wn}}
			+ \frac{\ln(16T/\delta)\sqrt{w}}{n}
			\right).
		\end{align*}
		Similarly,
		\begin{align*}
			\left\Vert
			\Phi p^t - \frac{1}{n} \sum_{i=1}^n \Phi \hat p_i^{t}
			\right\Vert_\infty =&\;
			\left\Vert
			\Phi 
			\right\Vert_\infty \cdot
			\left\Vert
			p^t - \frac{1}{n} \sum_{i=1}^n \hat p_i^t
			\right\Vert_\infty \\
			\leq&\; \frac{1}{\sqrt{w}} \cdot \sqrt{\frac{\ln(16dT/\delta)}{2n}} \\
			=&\; \sqrt{\frac{\ln(16dT/\delta)}{2wn}}
		\end{align*}
		where the inequality holds with probability at least $1-\frac{\delta}{8}$ in 
		every epoch $t$ by the definition of $\Phi$ and a Chernoff bound on the
		sampling error of $n$ samples from $\cP^t$, union bounded over $d$
		dictionary elements and $T$ epochs. By triangle inequality,
		\[
			\left\Vert
			\tilde y^{t} - \Phi p^t
			\right\Vert_\infty < c_\eps\left(
				\sqrt{\frac{\ln(16wT/\delta)}{nw}}
				+ \frac{\ln(16wT/\delta)\sqrt{w}}{n}
				\right)
			+ \sqrt{\frac{\ln(16dT/\delta)}{2wn}}
		\]
		andsince $c_\eps = \frac{e^b+1}{e^b-1} < \frac{b+2}{b}$, it is enough to
		set $b$ such that 
		\[
			\frac{b+2}{b}\left(\sqrt{\frac{\ln(16wT/\delta)}{nw}}
			+ \frac{\ln(16wT/\delta)\sqrt{w}}{n}\right)
			+ \sqrt{\frac{\ln(16dT/\delta)}{2wn}}	
			\leq 2(\log(T)+2)\sqrt{\frac{2\ln(16wnT/\delta)}{w\ell}}.
		\]
		Substituting in	our value
		\[
			b = \frac{2\left(\sqrt{\frac{\ln(16wT/\delta)}{nw}}
			+ \frac{\ln(16wT/\delta)\sqrt{w}}{n^2}\right)}
			{2(\log(T)+2)\sqrt{\frac{2\ln(16wnT/\delta)}{w\ell}}
			- 2\sqrt{\frac{\ln(16dT/\delta)}{2wn}}
			- \frac{\ln(16wT/\delta)\sqrt{w}}{n^2}}
		\]
		yields the claim.
\end{proof}

We'll need the following result to translate bounds on
$\left\Vert \tilde y^t - \Phi p^t \right\Vert$ into accuracy bounds relative to
$\cP^t$.

\begin{lemma}
\label{lem:translation}
	With probability at least $1-\frac{\delta}{8}$, in every epoch $t$ if
	\[
		\left\Vert
		\tilde y^{t} - \Phi p^t
		\right\Vert_\infty < B
	\]
	then, denoting by $\hat f^t$ the frequency oracle induced by $(\Phi, \tilde y^t)$,
	\[
		\max_{v \in [d]} \left\vert\hat f^t(v) - \cP^t(v)\right\vert
		\leq B\sqrt{w}
		+ 2\sqrt{\frac{\ln(16ndT/\delta)}{n}}.
	\]
\end{lemma}
\begin{proof}
	The outline of our proof is similar (and in some steps identical) to 	that of
	Theorem 2.5 	in~\citet{BS15}, but we provide it here for completeness. First,
	\begin{align*}
		\max_{v \in [d]} |\hat f^t(v) - \cP^t(v)|
		=&\; \max_{v \in [d]}\left\vert \langle
		\tilde y^t, \Phi e_v
		\rangle	- \langle
		p^t, e_v
		\rangle
		\right\vert \\
		=&\; \max_{v \in [d]}\left\vert
		\langle \tilde y^t - \Phi p^t, \Phi e_v \rangle
		+ \langle \Phi p^t,
		\Phi e_v\rangle - \langle p^t, e_v\rangle
		\right\vert \\
		\leq&\; \max_{v \in [d]}\left\vert
		\langle \tilde y^t - \Phi p^t, \Phi e_v \rangle
		\right\vert + \max_{v \in [d]}
		\left\vert \langle \Phi p^t,
		\Phi e_v\rangle - \langle p^t, e_v\rangle
		\right\vert.
	\end{align*}
	Then by Corollary 6.4 from~\citet{BLR13}, since
	$\gamma = \sqrt{\frac{\ln(16ndT/\delta)}{n}}$ and
	$w = 20n = \frac{20\ln(16ndT/\delta)}{\gamma^2}$, we get that with 
	probability at least $1-\frac{\delta}{8T}$ we can bound the second term by
	\[
		\max_{v \in [d]}
		\left\vert \langle \Phi p^t,
		\Phi e_v\rangle - \langle p^t, e_v\rangle
		\right\vert
		\leq \gamma\left(
		\left\Vert p^t \right\Vert_2^2
		+ \left\Vert e_v \right\Vert_2^2\right) \leq 2\gamma.
	\]
	By our assumption we can bound the first term by
	\[
		\left\Vert \tilde y^t - \Phi p^t \right\Vert_{\infty}
		\cdot \left\Vert \Phi e_v \right\Vert_1 \leq B\sqrt{w}.
	\]
	Together with a union bound over the $T$ epochs, these yield the claim.
\end{proof}

We can use these lemmas to prove an analogue of Corollary~\ref{lem:update}.
First, we specify our setting-specific redefinition of \emph{change}.

\begin{definition}
	We say a \emph{change} $\change{t}$ occurs in epoch $t \in [T]$ if
	there exists subgroup $j \in [m]$ such that $\cP_j^t \neq \cP_j^{t-1}$.
\end{definition}

This lets us state the necessary result.

\begin{lemma}
\label{lem:hh_update}
	With probability at least $1-\delta$, given adjacent changes $\change{t_1}$ and
	$\change{t_2}$, GlobalUpdate$^t \gets$ True in at most one epoch
	$t \in [t_1,t_2)$.
\end{lemma}
\begin{proof}
	The proof is identical to that of Lemma~\ref{lem:update}, only replacing
	Lemma~\ref{lem:local_error} with Lemma~\ref{lem:hh_local_error}.
\end{proof}

Lemmas~\ref{lem:adj_updates} and~\ref{lem:cap} apply in this setting
unmodified, which finally lets us prove the following accuracy guarantee.

\begin{theorem} [Accuracy Guarantee]
	With probability at least $1-\delta$, in every epoch $t \in [T]$ such that
	fewer than
	\[
		\frac{\eps}{4} \cdot
		\min\left(\frac{L}{8\sqrt{2n\ln(12mT/\delta)}}
		- 1,
		\frac{\log(T)\sqrt{\frac{n\ln(320n^2T/\delta)}{10\ell}}
		- \sqrt{\frac{\ln(16dT/\delta)}{10}}
		- 2\ln(320nT/\delta)\sqrt{\frac{5}{n}}}
		{\sqrt{\ln(320nT/\delta)}
		\left(1 + \frac{20}{\sqrt{n}}\right)}\right)
	\]
	changes have occurred in epochs $1,2, \ldots, t$,
	\[
		|\hat f^t(v) - \cP^t(v)| <
		4(\log(T) + 2)\sqrt{\frac{2\ln(320n^2T/\delta)}{\ell}} 
		+ \sqrt{\frac{\ln(\tfrac{16ndT}{\delta})}{n}}.
	\]
\end{theorem}
\begin{proof}
	The proof is nearly identical to that of Theorem~\ref{thm:acc}, replacing all
	lemmas with their heavy-hitter analogues proven above. We provide it here for
	completeness.

	First, with probability at least $1-\delta$ all of the preceding lemmas and 
	corollaries in this section apply, and we condition on them for the remainder
	of this proof. In particular, since
	\begin{align*}
		\min\left(\frac{\eps}{8a}, \frac{\eps}{4b}\right) =&\;
		\frac{\eps}{4} \cdot \min\left(\frac{1}{2a}, \frac{1}{b}\right) \\
		=&\; \frac{\eps}{4} \cdot
		\min\left(\frac{L-\frac{3}{\sqrt{2}}\sqrt{n\ln(12mT/\delta)}}
		{8\sqrt{2n\ln(12mT/\delta)}},
		\frac{2(\log(T)+2)\sqrt{\frac{2\ln(16wnT/\delta)}{w\ell}}
		- 2\sqrt{\frac{\ln(16dT/\delta)}{2wn}}
		- \frac{\ln(16wT/\delta)\sqrt{w}}{n^2}}
		{2\left(\sqrt{\frac{\ln(16wT/\delta)}{nw}}
		+ \frac{\ln(16wT/\delta)\sqrt{w}}{n^2}\right)}\right) \\
		>&\; \frac{\eps}{4} \cdot
		\min\left(\frac{L}{8\sqrt{2n\ln(12mT/\delta)}}
		- \frac{3}{16},
		\frac{(\log(T)+2)\sqrt{\frac{2\ln(320n^2T/\delta)}{20n\ell}}
		- \sqrt{\frac{\ln(16dT/\delta)}{40n^2}}
		- \frac{\ln(320nT/\delta)\sqrt{20n}}{n^2}}
		{\sqrt{\ln(320nT/\delta)}
		\left(\frac{1}{n} + \frac{20}{n^{3/2}}\right)}\right) \\
		=&\; \frac{\eps}{4} \cdot
		\min\left(\frac{L}{8\sqrt{2n\ln(12mT/\delta)}}
		- \frac{3}{16},
		\frac{(\log(T)+2)\sqrt{\frac{n\ln(320n^2T/\delta)}{10\ell}}
		- \frac{1}{2}\sqrt{\frac{\ln(16dT/\delta)}{10}}
		- 2\ln(320nT/\delta)\sqrt{\frac{5}{n}}}
		{\sqrt{\ln(320nT/\delta)}
		\left(1 + \frac{20}{\sqrt{n}}\right)}\right) \\
		>&\; \frac{\eps}{4} \cdot
		\min\left(\frac{L}{8\sqrt{2n\ln(12mT/\delta)}}
		- 1,
		\frac{\log(T)\sqrt{\frac{n\ln(320n^2T/\delta)}{10\ell}}
		- \sqrt{\frac{\ln(16dT/\delta)}{10}}
		- 2\ln(320nT/\delta)\sqrt{\frac{5}{n}}}
		{\sqrt{\ln(320nT/\delta)}
		\left(1 + \frac{20}{\sqrt{n}}\right)}\right)
	\end{align*}
	we know that the number of changes thus far is less than
	$\min\left(\frac{\eps}{8a}, \frac{\eps}{4b}\right)$, and by
	Lemma~\ref{lem:cap}	for every user $i$, $c_i^V < \eps/4$ and $c_i^E < \eps/4$,
	i.e. no user has exceeded their voting or estimation privacy budget.
	
	Now	suppose for contradiction that in epoch $t$
	\[
		|\Phi p^t - \tilde y^{t-1}| > 
		4(\log(T)+2)\sqrt{\frac{2\ln(16wnT/\delta)}{w\ell}}.
	\]
	By Lemma~\ref{lem:hh_estimate_error} this means GlobalUpdate$^t \gets$ False,
	so by Lemma~\ref{lem:vote_yes} for every subgroup $j \in [m]$ there
	exists user $i \in S_j$ such that
	\[	
		||\Phi \hat p_i^t - \Phi \hat p_i^{\last{t}}||_\infty \leq
		2(\log(T)+1)\sqrt{\frac{2\ln(16wnT/\delta)}{w\ell}}.
	\]
	Lemma~\ref{lem:hh_local_error} then implies that
	\[
		||\Phi p_{\subpop{i}}^t - \Phi p_{\subpop{i}}^{\last{t}}||_\infty
		\leq 2(\log(T) + 2)\sqrt{\frac{2\ln(16wnT/\delta)}{w\ell}}.
	\]
	Since this holds for every subgroup $j$, we get that
	\[
		||\Phi p^t - \Phi p^{\last{t}}||_\infty
		\leq 2(\log(T) + 2)\sqrt{\frac{2\ln(16wnT/\delta)}{w\ell}},
	\]
	and since GlobalUpdate$^t \gets$ False, by Lemma~\ref{lem:hh_estimate_error}
	this means that $\tilde y^t = \tilde y^{\last{t}}$, so
	\[
		||\tilde y^t - \Phi p^t|||_\infty
		\leq 2(\log(T) + 2)\sqrt{\frac{2\ln(16wnT/\delta)}{w\ell}}.
	\]
	Plugging this quantity into Lemma~\ref{lem:translation} as $B$ gives that
	for all $v \in [d]$
	\[
		|\hat f^t(v) - \cP^t(v)| <
		2(\log(T) + 2)\sqrt{\frac{2\ln(16wnT/\delta)}{\ell}} 
		+ \sqrt{\frac{\ln(\tfrac{16ndT}{\delta})}{n}}.
	\]
	Substituting $w = 20n$ yields the claim.
\end{proof}

\end{document}